\documentclass{article}

\usepackage[preprint]{neurips_2026}

\usepackage[utf8]{inputenc} % allow utf-8 input
\usepackage[T1]{fontenc}    % use 8-bit T1 fonts
\usepackage{hyperref}       % hyperlinks
\usepackage{url}            % simple URL typesetting
\usepackage{booktabs}       % professional-quality tables
\usepackage{amsfonts}       % blackboard math symbols
\usepackage{nicefrac}       % compact symbols for 1/2, etc.
\usepackage{microtype}      % microtypography
\usepackage{xcolor}         % colors
\usepackage{dirtytalk}
\usepackage{amsmath}
\usepackage{amsthm}
\usepackage{comment}
\usepackage{multirow}
\usepackage{array}
\usepackage{graphicx}
\usepackage{rotating}

\newcommand{\pmv}[2]{#1{\scriptsize$\,\pm\,$#2}}

\newtheorem{assumption}{Assumption}
\newtheorem{proposition}{Proposition}

\title{Target-confidence Recourse Using tSeTlin machines: TRUST}

\author{%
K. Darshana Abeyrathna, Sara El Mekkaoui, Nils Enric Canut Taugbøl, Anuja Vats \\
  Group Research and Development\\
  Det Norske Veritas (DNV)\\
  Oslo, Norway \\
  \texttt{\{Darshana.Abeyrathna.Kuruge, Sara.El.Mekkaoui\}@dnv.com} \\
   \texttt{\{Nils.Enric.Canut.Taugbol, anuja.vats\}@dnv.com} \\
}

\begin{document}

\maketitle

\begin{abstract}
Counterfactual explanations are widely used to provide algorithmic recourse in high-stakes decision-making systems. Most existing methods formulate recourse as a binary objective: find the smallest change to an input that flips the model’s decision. However, in many real-world settings, decision-makers rely not only on predicted labels but also on confidence thresholds and risk margins. Counterfactuals that barely cross a decision boundary are often fragile, difficult to justify, and unstable under noise or model variation.

In this paper, we propose \emph{Target-confidence Recourse Using tSeTlin machines (TRUST)}, a formulation in which users explicitly specify the desired prediction confidence for recourse. Rather than generating counterfactuals and assessing confidence post hoc, our framework directly searches for minimal changes that satisfy a user-defined confidence target, enabling principled comparison of recourse options along cost, confidence, and robustness.

We instantiate TRUST using a Probabilistic Tsetlin Machine (PTM) combined with Bayesian optimization. The probabilistic, clause-based structure of PTM enables fine-grained explanations that link prediction confidence to the stability of underlying decision rules. In particular, we show that counterfactuals satisfying identical rules can still differ significantly in reliability, depending on how securely they satisfy the probabilistic conditions of those rules, revealing whether a decision is supported by robust or fragile clause activations.

Experiments on synthetic and real-world datasets demonstrate that target-confidence counterfactuals produce more robust and interpretable recourse compared to boundary-based approaches, achieving perfect robustness on multiple benchmarks while maintaining low recourse cost (e.g., L2 distance of 0.10 on the Haberman dataset at 0.92 confidence). By explicitly controlling prediction confidence and exposing rule-level stability, the proposed framework generates counterfactuals that move beyond minimal validity toward stable, transparent, and actionable decision support in high-stakes settings.
\end{abstract}

\section{Introduction}

Machine learning models are increasingly used to support or automate high-stakes decisions in domains such as credit approval, insurance underwriting,  clinical risk stratification, and hiring, where regulatory frameworks such as the EU AI Act and GDPR impose obligations of explainability \cite{EUAIActArt13, wachter2017counterfactual}. When a machine learning model produces an adverse outcome, the affected individual is entitled to understand what could have been done differently, also known as algorithmic recourse. Counterfactual explanations (CEs) have emerged as the most practically and legally grounded forms for this, specifying the minimal change to an input that would have yielded a favorable decision \cite{wachter2017counterfactual}.  Since their formal introduction, CEs have been extended along many axes encouraging diversity \cite{mothilal2020explaining}, manifold-consistency \cite{joshi2019towards, pawelczyk2020learning, poyiadzi2020face}, causal faithfulness \cite{karimi2021algorithmic}, and actionability \cite{ustun2019actionable}. Despite this breadth, virtually all methods share the same optimization skeleton: minimize an input-change cost subject to a label flip,
\begin{equation}
  \min_{x'} d(x, x') \quad \text{s.t.} \quad M(x') \geq 0.5,
\label{eq:1}  
\end{equation}

where \(M:\mathcal{X} \to [0,1]\) is the model's predicted probability for the positive class. However this binary framing is technically fragile and misaligned with how high-stakes decision are made.

\textbf{Fragility of boundary-hugging counterfactuals:} Prior work shows that counterfactuals generated near decision boundaries are often fragile under perturbations and model updates \cite{pawelczyk2022adversarial, slack2021manipulated, upadhyay2021robust, pmlr-v162-dominguez-olmedo22a}. This has motivated robust CE methods \cite{dutta2022robust,hamman2023robust, jiang2023formalising}, many of which observe that higher-confidence counterfactuals tend to be more stable. Existing approaches encourage this indirectly through margin constraints \cite{dutta2022robust}, neighborhood-validity objectives \cite{hamman2023robust}, or probabilistic certificates \cite{pawelczyk2023probabilistic}. However, prediction confidence remains a post-hoc property rather than a user-specified optimization target.

\textbf{CE misalignment:} Real decision systems frequently apply explicit confidence thresholds. For example, a counterfactual achieving \(M(x')=0.52\) may technically satisfy Eq.~\ref{eq:1} yet remain unacceptable under practical approval policies and vulnerable to small perturbations. Existing methods therefore cannot directly answer questions such as: \textit{What is the minimal change required to achieve approval with at least 80\% confidence?} Nor do they support principled comparison between low-cost but fragile recourse and higher-confidence, more robust alternatives.

\textbf{Relationship to uncertainty-aware methods:} Prior uncertainty-aware approaches model confidence or uncertainty explicitly during CE generation. CLUE \cite{antoran2020getting} searches for inputs that reduce epistemic uncertainty, while \cite{pawelczyk2023probabilistic} derive probabilistic robustness certificates under model retraining. However, these methods either explain uncertainty post-hoc or certify robustness after generation; they do not allow users to specify a desired confidence level as part of the recourse objective itself.

\textbf{This work:}
We propose Target-confidence Recourse Using tSeTlin machines (TRUST), in which prediction confidence is an optimization constraint. Formally, given an input  \(x\) with \(M(x) < 0.5\), a target confidence \(\tau \in (0.5, 1)\), and tolerance \(\epsilon>0\), we seek\footnote{For simplicity of exposition, we consider factual instances from class 0 and counterfactuals from class 1; the approach is directly applicable in the reverse setting.}:
\begin{equation}
\min_{x'} d(x, x') \quad \text{s.t.} \quad |M(x') - \tau| \leq \epsilon.
\label{eq:2}
\end{equation}

Where objective \ref{eq:2} strictly generalizes \ref{eq:1}: setting \(\tau = 0.5\) recovers standard label-flip CEs. For \(\tau > 0.5\), it directly encodes risk thresholds and supports analysis across the full confidence spectrum. Importantly, solutions to \ref{eq:2} at a higher \(\tau \) lie strictly deeper in the positive class region, providing robustness by construction rather than by post-hoc filtering \cite{hamman2023robust, pawelczyk2023probabilistic}.
We instantiate this framework using a Probabilistic Tsetlin Machine (PTM) \cite{kuruge2024probabilistic} paired with Bayesian optimization. The choice of PTM over other probabilistic classifiers-calibrated neural networks, Gaussian processes, or isotonically calibrated ensembles-is crucial and deliberate as PTMs expose a symbolic clause-based structure in which the model's probability estimate decomposes into contributions from human-interpretable logical rules. This enables a second, distinct contribution:\emph{ clause-level counterfactual attribution} \ref{sec:CLA}, which traces confidence differences between competing recourse options back to changes in specific clause activating probabilities. This explanatory layer is entirely absent from prior CE methods. When a practitioner must choose between two counterfactuals that achieve the same \(\tau\) but at different costs, clause-level attribution explains \emph{mechanistically} why the additional cost is or is not justified - an explanatory capability that black-box probabilistic models cannot provide.

Our main contributions are:
\begin{itemize}
    \item \textbf{Target-confidence CEs:} We introduce formulation Eq. \ref{eq:2} as the first CE objective in which prediction confidence is a user-specified constraint, not an emergent property. We prove asymptotic optimality of a constrained Bayesian optimization search under mild regularity conditions (Appendix~\ref{sec:convergence}).
    \item \textbf{Clause-level counterfactual attribution:} We show how PTMs (Appendix~\ref{app:TM}) enable a novel explanatory layer that connects feature-level changes to confidence shifts through interpretable logical clauses, supporting principled comparison for multiple recourse options.
    \item \textbf{Robustness by design:} We establish a practical link between confidence-aware counterfactuals and robustness against noisy execution.
\end{itemize}

\section{Background and Related Work}
\label{sec:2}
We organize the literature along four axes that motivate our contribution: the counterfactual research, the robustness problem that boundary-based methods create, uncertainty-aware approaches to recourse, and the role of interpretable model structure in explanation quality.

\textbf{Counterfactual explanation methods:} Since \cite{wachter2017counterfactual} introduced counterfactual explanations (CEs) as a mechanism for GDPR-compliant recourse, many extensions have improved diversity \cite{mothilal2020explaining}, feasibility and actionability \cite{ustun2019actionable, poyiadzi2020face, laugel2018comparison}, manifold consistency \cite{joshi2019towards, pawelczyk2020learning}, and causal faithfulness \cite{karimi2021algorithmic}. Other works address non-differentiable models through probabilistic approximations or tree-structure search \cite{lucic2022focus, tolomei2017interpretable}, while generative approaches constrain recourse to semantically meaningful latent subspaces \cite{downs2020cruds}.

Despite their diversity, all of these methods share the same core objective: minimize input-change cost subject to a \textbf{label flip} constraint. This formulation treats all points on the positive side of the decision boundary equally valid for recourse and does not offer a mechanism to the practitioner to express or reason about confidence based preferences among them. Our method aims to fill this gap by allowing practitioners to specify and compare recourse options at user-defined confidence levels.

\textbf{Robustness for CEs:} Many works have shown that counterfactuals close to the decision boundary are fragile under perturbations and model changes \cite{pawelczyk2023probabilistic, slack2021manipulated}. Robustness has therefore been studied under several threat models, including model updates \cite{upadhyay2021robust, jiang2023formalising}, model multiplicity \cite{jiang2024argumentative}, noisy execution of recourse \cite{hamman2023robust}, and test-time perturbations \cite{artelt2021evaluating}. A consistent observation across these works is that higher-confidence counterfactuals tend to be more robust. Existing methods encourage this indirectly through margin constraints \cite{dutta2022robust}, neighborhood-validity objectives \cite{hamman2023robust}, or worst-case optimization \cite{jiang2023formalising}.

However, robustness in these approaches remains secondary to the standard label-flip objective (Eq.~\ref{eq:1}), with confidence treated as a regularizer or post-hoc evaluation criterion. Consequently, users cannot explicitly request recourse at a desired robustness or confidence level. Our formulation addresses this limitation by making prediction confidence a first-class optimization constraint, enabling robustness by construction rather than post-hoc filtering.

\textbf{Uncertainty-aware and probabilistic recourse:} A more recent line of work incorporates uncertainty into CE generation. Antorán et al.\cite{antoran2020getting} propose CLUE, which searches the latent space of a generative model for inputs that reduce a Bayesian Neural Network's epistemic uncertainty. %CLUE provides explanations of why a model is uncertain about a given input: the explanation target is the input instance, and the goal is a latent-space perturbation that makes the model more confident in its existing prediction. This is a fundamentally different problem from ours. In CLUE, there is no recourse objective, no user-specified target probability, and no cost-of-change minimization; the framework is not designed to suggest actionable input modifications that achieve a desired outcome at a desired confidence level.
Pawelczyk et al. \cite{pawelczyk2023probabilistic} provide probabilistic robustness certificates, characterizing the probability that a CE generated under one model remains valid when the model is retrained on new data. While these are important theoretical contributions, they address different questions: In CLUE, there is no recourse objective, no user-specified target probability, and no cost-of-change minimization; thus it not designed to suggest actionable input modifications that achieve a desired outcome at a desired confidence level while for \cite{pawelczyk2023probabilistic} it quantifies the reliability of a CE after generation, rather than allowing the user to request a CE at a specified confidence level. %The user receives a certificate; they cannot ask for recourse that meets a risk threshold by design.

Our framework is the first to treat prediction confidence as a user-specified optimization constraint in CE generation. This distinction \say{constraint versus certificate}, \say{design-time versus post-hoc} is the central technical contribution of this paper relative to both of the above lines of work.

\textbf{Interpretable models and clause-level attribution:} Most CE methods are model-agnostic and provide feature-level explanations of the form \say{change feature \(i\) by \(\delta_i\)} \cite{wachter2017counterfactual, mothilal2020explaining}. However, when multiple counterfactuals exist for the same instance, such explanations cannot clarify why small feature differences may produce substantially different confidence levels or robustness properties.

Interpretable model families partially address this limitation by exposing internal decision structure. Decision trees and rule lists relate recourse to branching conditions \cite{tolomei2017interpretable}, while linear models support coefficient-based attribution \cite{ustun2019actionable}. The Tsetlin Machine (TM) \cite{granmo2018tsetlin} and its probabilistic extension, the Probabilistic Tsetlin Machine (PTM) \cite{kuruge2024probabilistic}, further enable predictions to be decomposed into propositional logical clauses with probabilistic activations. This supports \emph{clause-level counterfactual attribution}, where confidence differences between counterfactuals can be traced to changes in clause firing probabilities, providing a mechanistic explanation of why one recourse option is more robust or costly than another.

To our knowledge, no prior CE method supports this type of mechanistic comparison across competing recourse options.

\section{Problem Formulation}
Let $M:\mathcal{X} \rightarrow [0,1]$ be a probabilistic binary classifier that outputs the predicted probability of the positive class. Given an input $x \in \mathcal{X}$ such that $M(x) < 0.5$, the goal of algorithmic recourse is to provide guidance on how $x$ could be changed to achieve a favorable outcome. The standard CE objective seeks the minimal-cost input modification that flips the predicted label:
\(
\min_{x'} d(x, x') \quad \text{s.t.} \quad M(x') \geq 0.5,
\)
where $d(\cdot,\cdot)$ is a distance or cost function. This enforces label validity while minimizing change. While intuitive, as discussed in Sec.\ref{sec:2} this objective implicitly assumes that any point on the positive side of the decision boundary provides equally meaningful recourse, an assumption easily violated in practice. %In practice, this assumption is often violated as predictions are rarely acted upon solely based on their sign. Instead, decision-makers rely on confidence thresholds, safety margins, or risk constraints. For example, a bank may approve loans only when the predicted probability of repayment exceeds a fixed threshold. A counterfactual that barely satisfies $M(x') \geq 0.5$ may be valid in a binary sense but unacceptable in practice. Such counterfactuals are also fragile: small execution noise or minor model updates can invalidate the recourse.

%Existing robustness-oriented CE methods partially address this issue by encouraging higher confidence, but they do so implicitly and without user control.

We therefore propose to formalize counterfactual recourse as a \emph{target-confidence} problem. Given a desired confidence level $\tau \in (0.5,1)$ and tolerance $\epsilon > 0$, we seek:
\begin{equation}
    \min_{x'} d(x, x') \quad \text{s.t.} \quad |M(x') - \tau| \leq \epsilon.
\end{equation}

This formulation makes confidence an explicit requirement rather than an emergent property. It allows users to:
(i) request recourse at different confidence levels,
(ii) compare alternative counterfactuals based on cost--confidence trade-offs, and
(iii) reason explicitly about robustness through confidence margins. This objective generalizes standard CEs as a special case when $\tau = 0.5$.

\section{Target-Confidence Counterfactual Framework}
\label{sec:4}

We now present the general framework for generating target-confidence counterfactual explanations.
The framework is model-agnostic and separates the definition of the counterfactual objective from the choice of predictive model and optimization strategy.

%\subsection{Design principles}
Our framework is guided by four design principles motivated by limitations identified in prior work.

\textbf{Explicit confidence control:}
Unlike standard CE formulations, the desired prediction confidence is specified \emph{before} counterfactual generation. This allows users or decision-makers to request recourse aligned with their risk tolerance.\\
\textbf{Cost--confidence trade-off:}
Higher confidence typically requires larger deviations from the original input. Rather than hiding this trade-off, the framework exposes it explicitly, enabling principled comparison of alternative recourse options.\\
\textbf{Robustness by construction:}
Prior work has shown strong connections between confidence margins and robustness under noisy execution and model changes \cite{dutta2022robust,hamman2023robust}. By targeting confidence directly, robustness emerges naturally rather than being enforced through auxiliary constraints.\\
\textbf{Comparability of counterfactuals:}
Generating counterfactuals at multiple confidence levels enables comparison not only between counterfactuals and the original instance, but also among counterfactuals themselves.

\subsection{Formal objective}

Let $M(x)$ denote the predicted probability of the positive class. Given an input $x$, a target confidence $\tau$, and a tolerance $\epsilon$, we define the feasible counterfactual set as:
\begin{equation}
\mathcal{C}_{\tau}(x)
=
\{x' \in \Omega(x) \mid |M(x') - \tau| \leq \epsilon \},
\end{equation}

where \(\Omega(x)\) denotes the set of feasible counterfactual candidates for the factual instance \(x\), including any domain, plausibility, and actionability constraints.

The target-confidence counterfactual problem is:
\begin{equation}
\min_{x' \in \mathcal{C}_{\tau}} d(x, x'),
\end{equation}
where $d(\cdot,\cdot)$ captures the cost or effort required to implement the recourse.\footnote{The objective can alternatively be formulated as a multi-objective optimization problem that jointly minimizes recourse cost and probability deviation, e.g., $\big(d(x,x'),\, |M(x')-\tau|\big)$, allowing Pareto-optimal trade-offs between proximity and confidence satisfaction.} It generalizes standard CE for $\tau = 0.5$. Importantly, it allows the generation of \emph{multiple} counterfactuals for different $\tau$ values, supporting scenario analysis and decision-making under uncertainty.

\begin{comment}
\subsection{Relation to existing robustness objectives}

Several robust CE methods indirectly encourage high-confidence solutions by pushing counterfactuals away from the decision boundary. For example, neighborhood-based robustness criteria~\cite{dutta2022robust,hamman2023robust} and probabilistic validity guarantees~\cite{pawelczyk2023probabilistic} rely on confidence as a proxy.

Our framework differs in that confidence is not a surrogate objective but a constraint. This distinction is crucial: instead of asking whether a counterfactual is robust \emph{after} generation, we ask which counterfactual achieves a desired robustness level \emph{by design}.    
\end{comment}

\section{Model Instantiation}
While the target-confidence framework \ref{sec:4} is model-agnostic, it requires a classifier that produces calibrated probability estimates and supports meaningful probabilistic interpretations of confidence differences across candidate counterfactuals. We instantiate it using a Probabilistic Tsetlin Machine (PTM) \cite{kuruge2024probabilistic} paired with Bayesian optimization (BO), for the reasons established below.
\subsection{Probabilistic Tsetlin Machine}
The PTM extends the Tsetlin Machine \cite{granmo2018tsetlin} by introducing stochasticity into clause evaluation, resulting in probabilistic predictions over a symbolic clause-based structure (see Appendix~\ref{app:TM} for technical details).

This structure is particularly important in the context of target-confidence counterfactuals. In our setting, multiple counterfactuals may achieve the same predicted label or even the same target confidence, but through different underlying mechanisms. In contrast to discriminative classifiers whose scores lack direct probabilistic interpretation without post-hoc calibration making  confidence differences difficult to attribute, PTMs enable clause-level explanations that link feature changes directly to shifts in confidence. This makes PTMs especially well-suited for supporting counterfactual comparison and justification, which is one of the central goals of this work (Sec.\ref{sec:CLA}).

\subsection{Bayesian optimization for counterfactual search}
The counterfactual search problem is non-convex and potentially expensive to evaluate, especially when predictions are stochastic. Bayesian optimization is well-suited to this setting, as it balances exploration and exploitation while minimizing the number of model queries.

We use BO to search over the input space for candidate counterfactuals whose predicted probability matches the target confidence. The acquisition function naturally trades off proximity to the original instance and satisfaction of the confidence constraint. Importantly, BO allows us to generate multiple counterfactual candidates at the same confidence level, enabling diversity and comparison without modifying the objective.

The proposed search is supported by theoretical guarantees: under standard regularity conditions, constrained Bayesian optimization is asymptotically optimal, ensuring eventual discovery of feasible target-confidence counterfactuals and convergence to the minimum-cost solution (see Appendix~\ref{sec:convergence}).

\section{Counterfactual Comparison via Clause-Level Attribution }
\label{sec:CLA}
When multiple counterfactuals exist for the same instance, whether at the same target confidence \(\tau\) or across different confidence levels; the practitioner faces a selection problem: which recourse option to pursue, and why. When PTM is used, two counterfactuals may differ only marginally in feature space yet achieve substantially different confidence levels, because they activate different subsets of the PTM's clause structure. This section formalizes how clause-level attribution resolves this ambiguity.\\
\textbf{Literal-level explanations:} At the literal level, counterfactual explanations describe which features changed and by how much. This level aligns with standard CE explanations and remains important for actionability.
However, literal explanations alone cannot explain why two counterfactuals that differ only slightly in feature values result in substantially different confidence levels.\\
\textbf{Clause-level counterfactual attribution:} Clause-level analysis addresses this limitation and is a key motivation for our use of the PTM over other machine leraning methods. PTM clauses capture class-specific sub-patterns as conjunctions of literals, and predictions arise from the aggregation of their activations. Consequently, changes in confidence can be traced back to variations in the firing probabilities of individual clauses and their literal composition (Sec. \ref{sec:cla_example} illustrates an example).

By comparing clause activation distributions between the original instance and multiple counterfactuals, we can quantify how each clause contributes to differences in prediction confidence. This makes it possible to explain not only why a counterfactual succeeds, but also why one counterfactual achieves higher confidence or robustness than another, and to communicate this to users, decision-makers, or anyone involved.

\begin{comment}
\subsection{Comparing counterfactuals}

When multiple counterfactuals exist for the same instance, our framework supports comparison along three dimensions:
(i) cost of change,
(ii) achieved confidence, and
(iii) underlying explanatory mechanisms.

Returning to the loan example, a small increase in income may yield marginal approval confidence by activating a single clause, while a larger increase in initial deposit may activate multiple independent clauses, resulting in higher confidence. Clause-level explanations make such differences explicit and communicable to both customers and decision-makers.

This comparison capability is largely absent from existing CE methods, which typically present counterfactuals in isolation.    
\end{comment}

\section{Experimental Evaluation}

The goal of the experimental evaluation is to assess whether target-confidence counterfactuals provide meaningful advantages over standard counterfactual explanations in terms of robustness, interpretability, and decision support. Rather than focusing on predictive accuracy, our experiments are designed to evaluate the quality and stability of algorithmic recourse. The implementation and all experimental code related to this section can be found at: \url{https://anonymous.4open.science/r/tm_counterfactual-6360/README.md}

\subsection{Experimental Setup}

We consider binary classification tasks where counterfactual explanations are commonly used for recourse.
Experiments span two synthetic and two real-world datasets, summarised in Table~\ref{tab:datasets}.

\paragraph{Synthetic data:} We generate two balanced datasets of 1\,500 samples each (seed\,=\,42). \emph{Synthetic~2D} draws class-conditional samples from Gaussians centred at $(-2,-2)$ and $(+2,+2)$ with covariance $0.5\,I_{2}$.
\emph{Synthetic~5D} combines Normal, Laplace, clipped Cauchy, Uniform, and Poisson marginals across five features. Both are split 80/20 into train/test sets with stratification; the training set
is further subsampled to 400 instances via stratified $k$-means to keep PTM
training tractable.

\paragraph{Real-world data:} \emph{Iris (binary)} ($n{=}150$, $d{=}4$) distinguishes setosa from non-setosa. \emph{Haberman Survival} ($n{=}306$, $d{=}3$) predicts five-year survival after breast-cancer surgery.
All features are MinMax-scaled to $[0,1]$ before splitting 80/20 with stratification.

\begin{table}[ht]
\centering
\caption{Dataset overview.}
\label{tab:datasets}
\setlength{\tabcolsep}{6pt}
\renewcommand{\arraystretch}{1.1}
\begin{tabular}{lcccc}
\toprule
\textbf{Dataset} & \textbf{Samples} & \textbf{Features} & \textbf{Train} & \textbf{Test} \\
\midrule
Synthetic 2D     & 1\,500 & 2  & 400 & 300 \\
Synthetic 5D     & 1\,500 & 5  & 400 & 300 \\
Iris (binary)    & 150    & 4  & 120 & 30  \\
Haberman Survival& 306    & 3  & 244 & 62  \\
\bottomrule
\end{tabular}
\end{table}

\paragraph{Query selection:}
For each dataset we randomly select $N_{\text{eval}}{=}10$ class-0 test instances as factual queries and seek counterfactuals that move each query toward class~1.

\paragraph{Baseline methods:}
We compare against two gradient-based and sampling-based counterfactual methods from the CARLA library~\cite{pawelczyk2021carla}: \emph{Wachter}~\cite{wachter2017counterfactual}, which minimises a weighted sum of classification loss and $L_1$ distance via gradient descent ($\lambda{=}0.5$, lr$\ {=}\,0.1$, 1\,000 iterations); and \emph{GrowingSpheres}~\cite{laugel2018comparison}, which searches for
counterfactuals on expanding hyperspheres around the factual. Both methods use a two-hidden-layer ANN (50--10 units, ReLU, softmax output, Adam, 250~epochs, early stopping) as the underlying classifier.

\paragraph{PTM+BO:}
The Probabilistic Tsetlin Machine is trained on binarized features obtained by thresholding each continuous feature at all unique training-set values (synthetic) or at up to 20 quantile-spaced thresholds per feature (real-world datasets with higher cardinality). Default PTM hyperparameters are clauses $C{=}20$, automaton states $S{=}100$, precision $s{=}1.5$, threshold $T{=}5$, and trained for 50--80 epochs. Counterfactual search is performed via multi-objective Bayesian Optimisation (Optuna TPE sampler) minimising $(|\hat{p}(x') - \tau|,\; \|x'-x\|_2)$ jointly, with a tolerance of $\varepsilon{=}0.1$ on the probability gap. Each BO run uses 300 trials; the PTM's \texttt{predict\_prob} function is valuated with 50 Monte Carlo clause samples per call during optimisation and 100~MC samples for final metric evaluation.
We report results at two confidence thresholds, $\tau{=}0.50$ (proximity-focused)
and $\tau{=}0.85$ (robustness-focused).

\paragraph{Uncertainty quantification:}
Because the PTM's stochastic clause sampling and the BO's random search introduce run-to-run variation, we repeat the BO search 10~times per query (with seeds $42, 43, \ldots, 51$) and report the mean$\,\pm\,$std of each metric.

\paragraph{Evaluation metrics:}
For each counterfactual we measure: \emph{L1 distance} (Manhattan) and \emph{L2 distance} (Euclidean) between the factual and counterfactual in the (MinMax-scaled) feature space; \emph{Confidence}, the classifier's predicted probability of the target class at the counterfactual; \emph{Robustness}, the fraction of 50 Gaussian-perturbed copies ($\sigma{=}0.01$) that retain the target-class prediction; and
\emph{Success Rate}, the fraction of queries that receive at least one valid counterfactual.

\paragraph{Computational environment:}
All experiments were conducted on a laptop with a 12th~Gen Intel Core i5-1245U CPU, 16\,GB RAM, and Intel Iris Xe integrated graphics (no discrete GPU).
Individual PTM+BO counterfactual searches (300 BO trials $\times$ 10 repeats per query) required approximately 1~minute per query on the lower-dimensional datasets (Synthetic~2D, Iris, Haberman) and up to 4~minutes per query on the Synthetic~5D dataset.

\subsection{Results and Analysis}

Table~\ref{tab:comparison} reports the counterfactual quality metrics for Wachter, GrowingSpheres, and PTM+BO at two confidence thresholds ($\tau{=}0.50$ and $\tau{=}0.85$) across all four benchmarks. PTM+BO metrics are reported as mean\,$\pm$\,std over 10 independent BO runs per query to quantify stochastic variation. All methods achieve a 100\% success rate, confirming that every query receives a valid counterfactual. PTM+BO ($\tau{=}0.50$) consistently produces the smallest L1 and L2 distances on the lower-dimensional settings (Synthetic 2D and Iris), indicating proximity-optimal counterfactuals, while PTM+BO ($\tau{=}0.85$) dominates on Confidence across every dataset and achieves the highest Robustness in three of four benchmarks---often by a substantial margin. The trade-off between the two thresholds is clear: a higher $\tau$ yields counterfactuals that are farther from the decision boundary in feature space yet much more stable under input perturbations and more confident in their predicted class. GrowingSpheres is competitive on L1/L2 distance in the Synthetic 5D setting but lags behind PTM+BO on Confidence and Robustness throughout. These results suggest that PTM+BO provides a flexible mechanism for controlling the proximity--robustness trade-off through the threshold parameter $\tau$, an advantage not available in the baseline CARLA methods.

\begin{table}[ht]
\centering   
\centering
\caption{%
  Comparison of counterfactual methods across synthetic (2D, 5D) and real-world
  (Iris, Haberman Survival) datasets.
  $\downarrow$~lower is better; $\uparrow$~higher is better.
  \textbf{Bold} denotes the best mean value per metric within each dataset.
  PTM+BO values are mean\,$\pm$\,std over 10 BO repeats per query.
}
\label{tab:comparison}
\setlength{\tabcolsep}{5pt}
\renewcommand{\arraystretch}{1.15}
\resizebox{\textwidth}{!}{
% Original format had an extra column: Succ. Rate
\begin{tabular}{llcccc} % originally llccccc
\toprule
\textbf{Dataset}
  & \textbf{Method}
  & \textbf{L1}$\downarrow$
  & \textbf{L2}$\downarrow$
  & \textbf{Confidence}$\uparrow$
  & \textbf{Robustness}$\uparrow$
  % & \textbf{Succ. Rate}
  \\
\midrule

\multirow{4}{*}{Synthetic 2D}
  & Wachter
    & 0.525 & 0.372 & 0.561 & 0.656 % & \textbf{1.0000}
    \\
  & GrowingSpheres
    & 0.521 & 0.370 & 0.524 & 0.792 % & \textbf{1.0000}
    \\
  & PTM+BO ($\tau{=}0.50$)
    & \textbf{\pmv{0.350}{0.046}} & \textbf{\pmv{0.304}{0.024}} & \pmv{0.452}{0.084} & \pmv{0.347}{0.328} % & \textbf{1.0000}
    \\
  & PTM+BO ($\tau{=}0.85$)
    & \pmv{0.468}{0.043} & \pmv{0.372}{0.024} & \textbf{\pmv{0.805}{0.063}} & \textbf{1.000} % & \textbf{1.0000}
    \\
\midrule

\multirow{4}{*}{Synthetic 5D}
  & Wachter
    & 0.768 & 0.367 & 0.536 & 0.508 % & \textbf{1.0000}
    \\
  & GrowingSpheres
    & \textbf{0.545} & \textbf{0.289} & 0.587 & 0.934 % & \textbf{1.0000}
    \\
  & PTM+BO ($\tau{=}0.50$)
    & \pmv{0.658}{0.228} & \pmv{0.375}{0.148} & \pmv{0.506}{0.085} & \pmv{0.523}{0.331} % & \textbf{1.0000}
    \\
  & PTM+BO ($\tau{=}0.85$)
    & \pmv{0.677}{0.129} & \pmv{0.372}{0.078} & \textbf{\pmv{0.809}{0.063}} & \textbf{1.000} % & \textbf{1.0000}
    \\
\midrule

\multirow{4}{*}{Iris}
  & Wachter
    & 0.907 & 0.471 & 0.768 & \textbf{0.970} % & \textbf{1.0000}
    \\
  & GrowingSpheres
    & 0.693 & 0.396 & 0.584 & 0.824 % & \textbf{1.0000}
    \\
  & PTM+BO ($\tau{=}0.50$)
    & \textbf{\pmv{0.637}{0.072}} & \textbf{\pmv{0.366}{0.046}} & \pmv{0.417}{0.098} & \pmv{0.216}{0.229} % & \textbf{1.0000}
    \\
  & PTM+BO ($\tau{=}0.85$)
    & \pmv{0.892}{0.115} & \pmv{0.507}{0.075} & \textbf{\pmv{0.777}{0.092}} & \pmv{0.925}{0.113} % & \textbf{1.0000}
    \\
\midrule

\multirow{4}{*}{Haberman}
  & Wachter
    & 0.393 & 0.346 & 0.501 & 0.750 % & \textbf{1.0000}
    \\
  & GrowingSpheres
    & 0.467 & 0.320 & 0.500 & 0.674 % & \textbf{1.0000}
    \\
  & PTM+BO ($\tau{=}0.50$)
    & \pmv{0.379}{0.257} & \pmv{0.277}{0.189} & \pmv{0.883}{0.041} & \textbf{1.000} % & \textbf{1.0000}
    \\
  & PTM+BO ($\tau{=}0.85$)
    & \textbf{\pmv{0.133}{0.062}} & \textbf{\pmv{0.099}{0.043}} & \textbf{\pmv{0.917}{0.034}} & \textbf{1.000} % & \textbf{1.0000}
    \\
\bottomrule
\end{tabular}
}
\end{table}
Our results show that boundary-based counterfactuals typically achieve minimal cost but remain close to the decision boundary, resulting in low confidence and poor robustness. Robust CE baselines improve stability but offer limited control over achieved confidence.

In contrast, target-confidence counterfactuals enable explicit control over prediction confidence. As the target confidence increases, counterfactuals move deeper into the positive region, leading to improved robustness under noise and model variation. Importantly, the framework exposes the trade-off between cost and robustness in a transparent manner.

Qualitative visualizations of the generated counterfactuals are provided in Appendix~\ref{app:results}, including 2D decision-boundary plots and PCA projections for higher-dimensional datasets. These visualizations support the quantitative findings by illustrating how higher target-confidence counterfactuals move further from unstable boundary regions into more robust areas of the learned decision landscape.

\subsection{Interpretability and Comparison of Counterfactuals}
\label{sec:cla_example}
To demonstrate the practical behavior of the proposed reliability-aware counterfactual explanation framework, we conducted experiments on the Haberman's Survival Dataset.

The nature of the PTM is that its rules are not fixed at the inference. One of the dominant rules learned during training was: \emph{Age $\le$ 73 AND Positive lymph nodes $\le$ 4} which supports survival beyond five years. Both counterfactuals generated at $\tau{=}0.50$ and $\tau{=}0.85$ for the factual patient, \emph{Age = 56, Operation year = 65, Positive lymph nodes = 9} satisfy this rule:
\begin{itemize}
    \item Counterfactual 1 (at $\tau{=}0.50$):
    
    \emph{Age = 56, Operation year = 65, Positive lymph nodes = 4}

    \item Counterfactual 2 (at $\tau{=}0.85$):
    
    \emph{Age = 56, Operation year = 65, Positive lymph nodes = 2}
\end{itemize}

However, their reliability differs due to rule fragility, which can be analyzed using automata-level inclusion and exclusion probabilities of their literals. A clause becomes fragile when a stricter version of one of its conditions may be included with non-negligible probability. 

For instance, there is approximately a $0.225$ probability that the lymph-node threshold in the above rule may tighten from $\le$ 4 to $\le$ 3. When this stricter condition becomes active, Counterfactual 1 no longer satisfies the rule, causing the clause output to collapse from active to inactive. This reduces the positive voting support for the survival class and introduces prediction instability. In contrast, the rule for the high-confidence counterfactual has lymph nodes at 2, placing the patient well within all active rule boundaries. As a result, the activation is consistent. 

This behavior reflects the decision mechanism of the PTM: after training, clauses are represented through learned probability distributions over literal-inclusion states rather than fixed deterministic rules. During inference, clause activation therefore depends on which literals are probabilistically included and whether the input satisfies them. Prediction robustness depends on how securely the input satisfies the high-probability literals of the relevant clauses, since small changes in feature values can affect clause outputs and the resulting vote aggregation.

Unlike traditional counterfactual explanations that focus only on feature changes, this approach additionally quantifies how stable the underlying rules remain after modification, enabling practitioners to distinguish between minimally valid and structurally robust counterfactuals.

\section{Discussion and Limitations}
\textbf{Relation to robustness:} Our framework aligns naturally with existing robustness notions. Target-confidence counterfactuals with higher $\tau$ values tend to lie farther from decision boundaries, improving stability under noisy execution. This connection has been observed empirically in prior work~\cite{dutta2022robust,hamman2023robust}, and our formulation provides a direct mechanism for exploiting it. However, our experiments evaluate only a limited set of target-confidence levels, and broader analysis across a wider range of $\tau$ values remains important future work.\\
\textbf{Model dependence:} While the framework is model-agnostic in principle, its effectiveness depends on meaningful probabilistic predictions. Poorly calibrated models may limit the interpretability of confidence-based recourse, making calibration an important future direction.\\
\textbf{Benchmark scope and evaluation:} Our empirical evaluation is limited to a relatively small set of synthetic and real-world datasets selected to demonstrate the core properties of target-confidence recourse and clause-level attribution. Broader validation across larger benchmarks remains future work. Direct comparison against CE methods is also challenging because many approaches are tightly coupled to specific model classes, optimization assumptions, or differentiability requirements. Furthermore, different predictive models may learn substantially different decision boundaries and confidence landscapes on the same dataset, making one-to-one comparisons inherently difficult.\\
\textbf{Computational considerations:} BO introduces additional computational overhead compared to gradient-based CE methods, although this is partly offset by reduced model queries and improved solution quality. PTM training can also be slower than conventional neural-network pipelines because the current implementation relies on a research-oriented PTM framework rather than newer optimized TM implementations supporting GPU acceleration and parallelized execution. However, once trained, PTM inference remains computationally lightweight and enables fast probabilistic prediction generation suitable for real-time counterfactual evaluation and uncertainty estimation \cite{mao2025dynamic, duan2025ethereal}.\\
\textbf{Why structured models matter for counterfactual comparison:} While target-confidence counterfactuals can, in principle, be generated using any probabilistic classifier, our results highlight the importance of structured models for explanation quality. In particular, the clause-based nature of PTMs enables explanations that go beyond feature-level differences and provide insight into the mechanisms driving confidence changes. This becomes especially important when multiple counterfactuals are available and stakeholders must decide which recourse option to pursue.

\section{Conclusion}

We presented TRUST, a target-confidence formulation of counterfactual recourse that treats prediction confidence as an explicit requirement rather than a post-hoc property. This reframes recourse from merely crossing a decision boundary to finding actionable changes that meet a desired confidence level, exposing the trade-off between cost, confidence, and robustness.

Instantiated with a Probabilistic Tsetlin Machine and Bayesian optimization, TRUST generates counterfactuals across confidence levels and compares them by proximity and stability. Experiments show that higher-confidence counterfactuals generally yield more robust recourse under input perturbations, while making the cost of stronger decision support explicit.

The PTM-based instantiation further enables clause-level attribution, explaining not only what feature changes are required, but also why one counterfactual is more reliable than another by revealing whether the supporting rules are robustly or fragily activated. Overall, TRUST provides a confidence-aware and interpretable foundation for algorithmic recourse in high-stakes settings.

\bibliographystyle{plainnat} % Or another consistent style
\bibliography{sample-base}

@String{Computer = "{IEEE} Computer" }

@String{Springer = "Springer-Verlag" }

@article{wachter2017counterfactual,
  title     = {Counterfactual Explanations without Opening the Black Box: Automated Decisions and the GDPR},
  author    = {Wachter, Sandra and Mittelstadt, Brent and Russell, Chris},
  journal   = {Harvard Journal of Law \& Technology},
  volume    = {31},
  number    = {2},
  pages     = {841--887},
  year      = {2017}
}

@inproceedings{pawelczyk2022adversarial,
  title     = {Exploring Counterfactual Explanations through the Lens of Adversarial Examples: A Theoretical and Empirical Analysis},
  author    = {Pawelczyk, Martin and Agarwal, Chirag and Joshi, Shalmali and Upadhyay, Sohini and Lakkaraju, Himabindu},
  booktitle = {Proceedings of the International Conference on Artificial Intelligence and Statistics (AISTATS)},
  year      = {2022}
}

@inproceedings{pawelczyk2020learning,
  title={Learning model-agnostic counterfactual explanations for tabular data},
  author={Pawelczyk, Martin and Broelemann, Klaus and Kasneci, Gjergji},
  booktitle={Proceedings of the web conference 2020},
  pages={3126--3132},
  year={2020}
}

@inproceedings{poyiadzi2020face,
  title={FACE: feasible and actionable counterfactual explanations},
  author={Poyiadzi, Rafael and Sokol, Kacper and Santos-Rodriguez, Raul and De Bie, Tijl and Flach, Peter},
  booktitle={Proceedings of the AAAI/ACM Conference on AI, Ethics, and Society},
  pages={344--350},
  year={2020}
}

@inproceedings{karimi2021algorithmic,
  title={Algorithmic recourse: from counterfactual explanations to interventions},
  author={Karimi, Amir-Hossein and Sch{\"o}lkopf, Bernhard and Valera, Isabel},
  booktitle={Proceedings of the 2021 ACM conference on fairness, accountability, and transparency},
  pages={353--362},
  year={2021}
}

@inproceedings{ustun2019actionable,
  title={Actionable recourse in linear classification},
  author={Ustun, Berk and Spangher, Alexander and Liu, Yang},
  booktitle={Proceedings of the conference on fairness, accountability, and transparency},
  pages={10--19},
  year={2019}
}

@article{antoran2020getting,
  title={Getting a clue: A method for explaining uncertainty estimates},
  author={Antor{\'a}n, Javier and Bhatt, Umang and Adel, Tameem and Weller, Adrian and Hern{\'a}ndez-Lobato, Jos{\'e} Miguel},
  journal={arXiv preprint arXiv:2006.06848},
  year={2020}
}

@legaltext{EUAIActArt13,
  author = {{European Parliament} and {Council of the European Union}},
  title = {Regulation (EU) 2024/1689 of the European Parliament and of the Council of 13 June 2024 on harmonised rules on artificial intelligence (Artificial Intelligence Act)},
  journal = {Official Journal of the European Union},
  year = {2024},
  url = {http://data.europa.eu/eli/reg/2024/1689/oj},
  note = {Article 13 - Transparency and provision of information to deployers}
  }

@inproceedings{mothilal2020explaining,
  title={Explaining machine learning classifiers through diverse counterfactual explanations},
  author={Mothilal, Ramaravind K and Sharma, Amit and Tan, Chenhao},
  booktitle={Proceedings of the 2020 conference on fairness, accountability, and transparency},
  pages={607--617},
  year={2020}
}

@article{joshi2019towards,
  title={Towards realistic individual recourse and actionable explanations in black-box decision making systems},
  author={Joshi, Shalmali and Koyejo, Oluwasanmi and Vijitbenjaronk, Warut and Kim, Been and Ghosh, Joydeep},
  journal={arXiv preprint arXiv:1907.09615},
  year={2019}
}

@inproceedings{upadhyay2021robust,
  title     = {Towards Robust and Reliable Algorithmic Recourse},
  author    = {Upadhyay, Sohini and Joshi, Shalmali and Lakkaraju, Himabindu},
  booktitle = {Advances in Neural Information Processing Systems (NeurIPS)},
  year      = {2021}
}

@inproceedings{dutta2022robust,
  title     = {Robust Counterfactual Explanations for Tree-Based Ensembles},
  author    = {Dutta, Sanghamitra and Long, Jason and Mishra, Saumitra and Tilli, Cecilia and Magazzeni, Daniele},
  booktitle = {International Conference on Machine Learning (ICML)},
  year      = {2022}
}

@article{pawelczyk2021carla,
  title={Carla: a python library to benchmark algorithmic recourse and counterfactual explanation algorithms},
  author={Pawelczyk, Martin and Bielawski, Sascha and Heuvel, Johannes van den and Richter, Tobias and Kasneci, Gjergji},
  journal={arXiv preprint arXiv:2108.00783},
  year={2021}
}

@inproceedings{hamman2023robust,
  title     = {Robust Counterfactual Explanations for Neural Networks with Probabilistic Guarantees},
  author    = {Hamman, Faisal and Noorani, Erfaun and Mishra, Saumitra and Magazzeni, Daniele and Dutta, Sanghamitra},
  booktitle = {International Conference on Machine Learning (ICML)},
  year      = {2023}
}

@InProceedings{pmlr-v162-dominguez-olmedo22a,
  title = 	 {On the Adversarial Robustness of Causal Algorithmic Recourse},
  author =       {Dominguez-Olmedo, Ricardo and Karimi, Amir H and Sch{\"o}lkopf, Bernhard},
  booktitle = 	 {Proceedings of the 39th International Conference on Machine Learning},
  pages = 	 {5324--5342},
  year = 	 {2022},
  editor = 	 {Chaudhuri, Kamalika and Jegelka, Stefanie and Song, Le and Szepesvari, Csaba and Niu, Gang and Sabato, Sivan},
  volume = 	 {162},
  series = 	 {Proceedings of Machine Learning Research},
  month = 	 {17--23 Jul},
  publisher =    {PMLR},
  url = 	 {https://proceedings.mlr.press/v162/dominguez-olmedo22a.html},
}

@inproceedings{kuruge2024probabilistic,
  title={The probabilistic tsetlin machine: A novel approach to uncertainty quantification},
  author={Kuruge, Darshana Abeyrathna and El Mekkaoui, Sara and Hafver, Andreas and Agrell, Christian},
  booktitle={Proceedings of the 2024 8th International Conference on Advances in Artificial Intelligence},
  pages={39--47},
  year={2024}
}

@inproceedings{jiang2023formalising,
  title     = {Formalising the Robustness of Counterfactual Explanations for Neural Networks},
  author    = {Jiang, Junqi and Leofante, Francesco and Rago, Antonio and Toni, Francesca},
  booktitle = {Proceedings of the AAAI Conference on Artificial Intelligence (AAAI)},
  year      = {2023}
}

@article{pawelczyk2023probabilistic,
  title     = {Probabilistically Robust Recourse: Navigating the Trade-offs between Costs and Robustness in Algorithmic Recourse},
  author    = {Pawelczyk, Martin and Datta, Teresa and van den Heuvel, Johannes and Kasneci, Gjergji and Lakkaraju, Himabindu},
  journal   = {Proceedings of the International Conference on Learning Representations (ICLR)},
  year      = {2023}
}

@inproceedings{jiang2024argumentative,
  title     = {Recourse under Model Multiplicity via Argumentative Ensembling},
  author    = {Jiang, Junqi and Rago, Antonio and Leofante, Francesco and Toni, Francesca},
  booktitle = {International Conference on Autonomous Agents and Multiagent Systems (AAMAS)},
  year      = {2024}
}

@inproceedings{slack2021manipulated,
  title     = {Counterfactual Explanations Can Be Manipulated},
  author    = {Slack, Dylan and Hilgard, Anna and Lakkaraju, Himabindu and Singh, Sameer},
  booktitle = {Advances in Neural Information Processing Systems (NeurIPS)},
  year      = {2021}
}

@inproceedings{artelt2021evaluating,
  title     = {Evaluating Robustness of Counterfactual Explanations},
  author    = {Artelt, André and Vaquet, Valerie and Velioglu, Riza and Hinder, Fabian and Brinkrolf, Johannes and Schilling, Malte and Hammer, Barbara},
  booktitle = {IEEE Symposium Series on Computational Intelligence (SSCI)},
  year      = {2021}
}

@inproceedings{laugel2018comparison,
  title={Comparison-based inverse classification for interpretability in machine learning},
  author={Laugel, Thibault and Lesot, Marie-Jeanne and Marsala, Christophe and Renard, Xavier and Detyniecki, Marcin},
  booktitle={International conference on information processing and management of uncertainty in knowledge-based systems},
  pages={100--111},
  year={2018},
  organization={Springer}
}

@inproceedings{lucic2022focus,
  title={FOCUS: Flexible optimizable counterfactual explanations for tree ensembles},
  author={Lucic, Ana and Oosterhuis, Harrie and Haned, Hinda and De Rijke, Maarten},
  booktitle={Proceedings of the AAAI conference on artificial intelligence},
  volume={5.36},
  pages={5313--5322},
  year={2022}
}

@inproceedings{tolomei2017interpretable,
  title={Interpretable predictions of tree-based ensembles via actionable feature tweaking},
  author={Tolomei, Gabriele and Silvestri, Fabrizio and Haines, Andrew and Lalmas, Mounia},
  booktitle={Proceedings of the 23rd ACM SIGKDD international conference on knowledge discovery and data mining},
  pages={465--474},
  year={2017}
}

@article{downs2020cruds,
  title={Cruds: Counterfactual recourse using disentangled subspaces},
  author={Downs, Michael and Chu, Jonathan L and Yacoby, Yaniv and Doshi-Velez, Finale and Pan, Weiwei},
  journal={ICML WHI},
  volume={2020},
  pages={1--23},
  year={2020}
}

@article{granmo2018tsetlin,
  title={The Tsetlin Machine--A Game Theoretic Bandit Driven Approach to Optimal Pattern Recognition with Propositional Logic},
  author={Granmo, Ole-Christoffer},
  journal={arXiv preprint arXiv:1804.01508},
  year={2018}
}

@BOOK{test,
   author = "Donald E. Knuth",
   title = "Seminumerical Algorithms",
   volume = 2,
   series = "The Art of Computer Programming",
   publisher = "Addison-Wesley",
   address = "Reading, MA",
   edition = "2nd",
   month = "10~" # jan,
   year = "1981",
}

@article{Bull2011,
  author = {Bull, Adam D.},
  title = {Convergence Rates of Efficient Global Optimization Algorithms},
  journal = {Journal of Machine Learning Research},
  volume = {12},
  pages = {2879--2904},
  year = {2011}
}

@inproceedings{Srinivas2010,
  author = {Srinivas, Niranjan and Krause, Andreas and Kakade, Sham and Seeger, Matthias},
  title = {Gaussian Process Optimization in the Bandit Setting: No Regret and Experimental Design},
  booktitle = {Proceedings of ICML},
  year = {2010}
}

@inproceedings{Gardner2014,
  author = {Gardner, Jacob and Kusner, Matt and Xu, Zhixiang and Weinberger, Kilian and Cunningham, John},
  title = {Bayesian Optimization with Inequality Constraints},
  booktitle = {Proceedings of ICML},
  year = {2014}
}

@article{Gelbart2014,
  author = {Gelbart, Michael A. and Snoek, Jasper and Adams, Ryan P.},
  title = {Bayesian Optimization with Unknown Constraints},
  journal = {arXiv preprint arXiv:1403.5607},
  year = {2014}
}

@article{abeyrathna2021extending,
  title={Extending the tsetlin machine with integer-weighted clauses for increased interpretability},
  author={Abeyrathna, K Darshana and Granmo, Ole-Christoffer and Goodwin, Morten},
  journal={IEEE Access},
  volume={9},
  pages={8233--8248},
  year={2021},
  publisher={IEEE}
}

@article{darshana2020regression,
  title={The regression Tsetlin machine: a novel approach to interpretable nonlinear regression},
  author={Darshana Abeyrathna, K and Granmo, Ole-Christoffer and Zhang, Xuan and Jiao, Lei and Goodwin, Morten},
  journal={Philosophical Transactions of the Royal Society A: Mathematical, Physical and Engineering Sciences},
  volume={378},
  number={2164},
  year={2020},
  publisher={The Royal Society}
}

@article{mao2025dynamic,
  title={Dynamic Tsetlin machine accelerators for on-chip training using FPGAs},
  author={Mao, Gang and Rahman, Tousif and Maheshwari, Sidharth and Pattison, Bob and Shao, Zhuang and Shafik, Rishad and Yakovlev, Alex},
  journal={IEEE Transactions on Circuits and Systems I: Regular Papers},
  year={2025},
  publisher={IEEE}
}

@inproceedings{duan2025ethereal,
  title={ETHEREAL: Energy-efficient and High-throughput Inference using Compressed Tsetlin Machine},
  author={Duan, Shengyu and Shafik, Rishad and Yakovlev, Alex},
  booktitle={2025 10th International Workshop on Advances in Sensors and Interfaces (IWASI)},
  pages={1--6},
  year={2025},
  organization={IEEE}
}

@article{bergstra2011algorithms,
  title={Algorithms for hyper-parameter optimization},
  author={Bergstra, James and Bardenet, R{\'e}mi and Bengio, Yoshua and K{\'e}gl, Bal{\'a}zs},
  journal={Advances in neural information processing systems},
  volume={24},
  year={2011}
}

\appendix

\section{Theoretical Convergence Properties}
\label{sec:convergence}

To establish asymptotic guarantees for the discovery and convergence of target-confidence counterfactuals, we rely on the established properties of Bayesian optimization (BO) under constraints.

\begin{assumption}[Regularity and Strategy]
\label{assum:bo_convergence}
We assume the following conditions hold:
\begin{itemize}
    \item \textbf{(A1) Continuity:} The probabilistic prediction function $M:\mathcal{X}\rightarrow[0,1]$ and the cost function $d(x, \cdot)$ are continuous on the compact input space $\mathcal{X}$.
    \item \textbf{(A2) Feasibility:} The feasible counterfactual set $C_\tau = \{x' \in \mathcal{X} : |M(x') - \tau| \leq \epsilon\}$ has a nonempty interior.
    \item \textbf{(A3) Exploration:} The search uses a globally exploring acquisition strategy (e.g., Expected Improvement or GP-UCB) with known convergence rates \cite{bergstra2011algorithms, Bull2011,Srinivas2010}.
    \item \textbf{(A4) Constraint Handling:} Feasibility is managed via a constrained BO strategy that learns the feasible region during optimization \cite{Gardner2014,Gelbart2014}.
\end{itemize}
\end{assumption}

\begin{proposition}[Asymptotic Optimality]
Under Assumption \ref{assum:bo_convergence}, the sequence of candidate counterfactuals $\{x'_n\}$ generated by the BO algorithm satisfies:
\begin{enumerate}
    \item Feasible target-confidence counterfactuals are eventually identified (i.e., $\exists n$ such that $x'_n \in C_\tau$).
    \item The observed minimum cost converges to the true global optimum: $d^*_n \to d^*$ as $n \to \infty$, where $d^* = \min_{x' \in C_\tau} d(x, x')$.
\end{enumerate}
\end{proposition}

\begin{proof}
The proof is structured in three phases:

\textbf{Phase 1: Existence of an Optimal Solution.} 
By (A1), the function $x' \mapsto |M(x') - \tau|$ is continuous. Consequently, the feasible set $C_\tau$ is a closed subset of the compact space $\mathcal{X}$, making $C_\tau$ itself compact. Since $d(x, \cdot)$ is continuous on this compact set, the \textit{Extreme Value Theorem} guarantees that a global minimum $x^* \in C_\tau$ with cost $d^*$ exists.

\textbf{Phase 2: Discovery of the Feasible Region.} 
By (A2), $C_\tau$ contains an open neighborhood. Standard global exploration results for BO strategies like GP-UCB and Expected Improvement guarantee that the algorithm will generate query points arbitrarily close to any point in the search domain over time. Since $C_\tau$ has a nonempty interior, the algorithm is guaranteed to eventually sample at least one point $x'_n$ within the feasible region.

\textbf{Phase 3: Convergence to Minimum Cost.} 
Once the feasible region is identified, constrained BO strategies focus the search to minimize the objective while maintaining feasibility. As the acquisition strategy achieves sublinear regret, the sequence of feasible samples $\{x'_k \in C_\tau\}$ approaches the global optimum $x^*$. By the continuity of the distance function $d(x, \cdot)$ (A1), the sequence of best-observed costs $d^*_n$ converges to the optimal cost $d^*$.
\end{proof}

\textit{Note: While this result is general, it applies to the Probabilistic Tsetlin Machine (PTM) because its stochastic inference mechanism generates the continuous probability estimates required to satisfy the regularity assumptions in (A1) [5].}

\section{Technical Background on Tsetlin Machines} \label{app:TM}
This appendix provides a technical foundation for the Tsetlin Machine (TM) and the Probabilistic Tsetlin Machine (PTM), which serve as the core predictive components of the TRUST framework.

\subsection{The Tsetlin Machine (TM)}
The Tsetlin Machine is a learning automaton-based system that represents patterns using propositional logic \cite{granmo2018tsetlin}. 

\subsubsection{Input Representation and Literals}
An input vector $\mathbf{X} = (x_1, \dots, x_o) \in \{0, 1\}^o$ is expanded into a set of literals $L$ consisting of the original variables and their negations:
\begin{equation}
L = \{x_1, \dots, x_o, \neg x_1, \dots, \neg x_o\} = \{l_1, \dots, l_{2o}\} \text{ \cite{abeyrathna2021extending}}
\end{equation}

\subsubsection{Conjunctive Clauses}
The TM uses $m$ conjunctive clauses $C_j$. Each clause $C_j(X)$ is formed by ANDing a subset of literals $I_j \subseteq \{1, \dots, 2o\}$:
\begin{equation}
C_j(X) = \bigwedge_{k \in I_j} l_k \text{ \cite{darshana2020regression}}
\end{equation}
A clause outputs 1 if and only if all included literals are 1; otherwise, it outputs 0. If a clause is empty ($I_j = \emptyset$), it outputs 1 during learning and 0 during inference \cite{granmo2018tsetlin}.

\subsubsection{Tsetlin Automata (TA) Mechanism}
For each clause, a team of $2o$ Tsetlin Automata decides which literals to include. A TA is a finite-state machine with $2N$ states \cite{granmo2018tsetlin}. Let $a_{j,k} \in \{1, \dots, 2N\}$ be the state of the TA for literal $k$ in clause $j$:
\begin{itemize}
    \item \textbf{Exclude (Action 1):} If $1 \le a_{j,k} \le N$, the literal $l_k$ is excluded.
    \item \textbf{Include (Action 2):} If $N+1 \le a_{j,k} \le 2N$, the literal $l_k$ is included \cite{darshana2020regression}.
\end{itemize}

\subsubsection{Classification via Voting}
Clauses are divided into positive polarity ($C_j^+$), which learns patterns for class-1, and negative polarity ($C_j^-$), which learns patterns for class-0. The final prediction $\hat{y}$ is determined by a majority vote $v$:
\begin{equation}
v = \sum_{j=1}^{m/2} C_j^+(X) - \sum_{j=1}^{m/2} C_j^-(X) 
\end{equation}
\begin{equation}
\hat{y} = \begin{cases} 1 & \text{if } v \ge 0 \\ 0 & \text{if } v < 0 \end{cases}
\end{equation}

The basic Tsetlin Machine structure is illustrated in Figure~\ref{fig:tmmodel}.

\begin{figure}[!ht]
  \centering
  \includegraphics[width=0.6\textwidth]{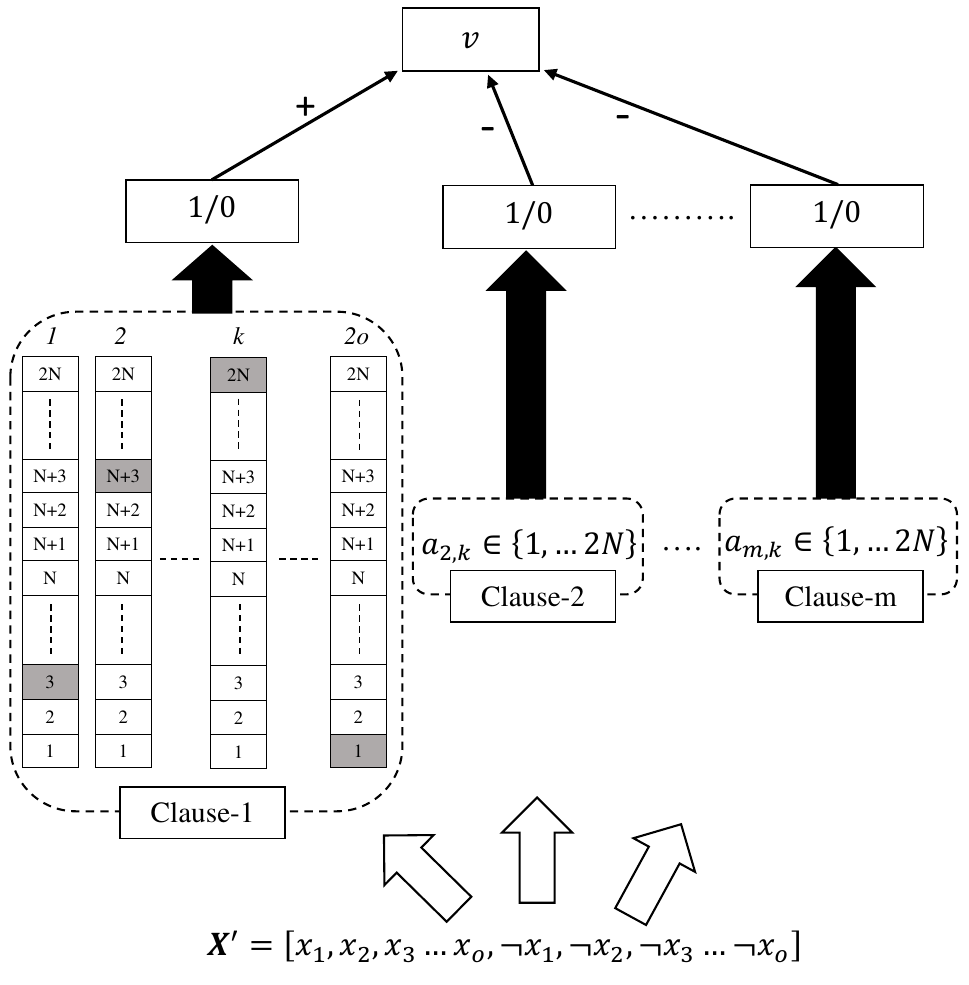}
  \caption{The Tsetlin Machine structure \cite{abeyrathna2021extending}}
  \label{fig:tmmodel}
\end{figure}

\subsubsection{Learning Feedback}
Learning is driven by two decentralized feedback types that adjust TA states based on the clause output, literal value, and class label \cite{darshana2020regression}. The hyperparameter $s$ controls the granularity of the clause.

\textbf{Type I Feedback (Combats False Negatives):}
This feedback is given to clauses with positive polarity when the target label $y = 1$. It reinforces Include actions for true literals and Exclude actions for false literals to help the clause evaluate to 1.

\begin{table}[!t]
\centering
% \begin{threeparttable}
\newcolumntype{P}[1]{>{\centering\arraybackslash}p{#1}}

\caption{Reward, Inaction and Penalty probabilities for Type I Feedback.}
\label{tab:type1}

\begin{tabular}{P{8mm}|c|c|c|c|c|c}
\toprule
% \multicolumn{3}{c|}{} & \multicolumn{4}{c}{Type I Feedback} \\ \hline
\multicolumn{3}{c|}{Clause Output} & \multicolumn{2}{c|}{1} & \multicolumn{2}{c}{0} \\ \hline
\multicolumn{3}{c|}{Literal Value} & 1 & 0 & 1 & 0 \\ \hline

\multirow{6}{*}{\rotatebox[origin=c]{90}{Current State}} 
& \multirow{3}{*}{Include} 
& Reward Probability   & (s-1)/s & NA & 0 & 0 \\
&                     
& Inaction Probability & 1/s     & NA & (s-1)/s & (s-1)/s \\
&                     
& Penalty Probability  & 0       & NA & 1/s & 1/s \\ \cline{2-7}

& \multirow{3}{*}{Exclude} 
& Reward Probability   & 0 & 1/s & 1/s & 1/s \\
&                     
& Inaction Probability & 1/s & (s-1)/s & (s-1)/s & (s-1)/s \\
&                     
& Penalty Probability  & (s-1)/s & 0 & 0 & 0 \\

\bottomrule
\end{tabular}
% \end{threeparttable}
\end{table}

\textbf{Type II Feedback (Combats False Positives):}
This feedback is given to clauses with negative polarity when $y = 1$. It forces the clause to evaluate to 0 by penalizing the exclusion of 0-valued literals.

\begin{table}[!t]
\centering
% \begin{threeparttable}
\newcolumntype{P}[1]{>{\centering\arraybackslash}p{#1}}

\caption{Reward, Inaction and Penalty probabilities for Type II Feedback.}
\label{tab:type2}

\begin{tabular}{P{8mm}|c|c|c|c|c|c}
\toprule
% \multicolumn{7}{c}{Type II Feedback} \\ \hline
\multicolumn{3}{c|}{Clause Output} & \multicolumn{2}{c|}{1} & \multicolumn{2}{c}{0} \\ \hline
\multicolumn{3}{c|}{Literal Value} & 1 & 0 & 1 & 0 \\ \hline

\multirow{6}{*}{\rotatebox[origin=c]{90}{Current State}} 
& \multirow{3}{*}{Include} 
& Reward Probability   & 0/s & NA & 0 & 0 \\
&                     
& Inaction Probability & 1     & NA & 1 & 1 \\
&                     
& Penalty Probability  & 0       & NA & 0 & 0 \\ \cline{2-7}

& \multirow{3}{*}{Exclude} 
& Reward Probability   & 0 & 0 & 0 & 0 \\
&                     
& Inaction Probability & 1 & 0 & 1 & 1 \\
&                     
& Penalty Probability  & 0 & 1 & 0 & 0 \\

\bottomrule
\end{tabular}
% \end{threeparttable}
\end{table}

\subsection{The Probabilistic Tsetlin Machine (PTM)}
The PTM \cite{kuruge2024probabilistic} extends the TM to quantify uncertainty by replacing deterministic integer states with State Probability Vectors (SPVs).

\subsubsection{State Representation}
Instead of a single integer $a_{j,k}$, each TA's state is a distribution over all $2N$ states:
\begin{equation}
SPV_{j,k} = [p(S_1), p(S_2), \dots, p(S_{2N})] \in^{2N} \text{ \cite{kuruge2024probabilistic}}
\end{equation}

\subsubsection{Learning with Transition Probability Matrices (TPM)}
During training, the SPVs are updated by multiplying them with TPMs derived from the feedback tables above:
\begin{equation}
SPV_{j,k}(n+1) \leftarrow SPV_{j,k}(n) \cdot TPM \text{ \cite{kuruge2024probabilistic}}
\end{equation}
This allows the model to learn the probability of a literal being included rather than a binary decision.

\subsubsection{Stochastic Inference}
During inference, states are sampled from their respective SPVs \cite{kuruge2024probabilistic}. By performing $K$ stochastic passes, the PTM generates a predictive mean that serves as a calibrated probability estimate $M(x)$:
\begin{equation}
M(x) = \frac{1}{K} \sum_{j=1}^{K} p_j(Y|X)
\end{equation}

\section{Additional Results} \label{app:results}

\begin{figure}[!ht]
  \centering
  \includegraphics[width=0.8\textwidth]{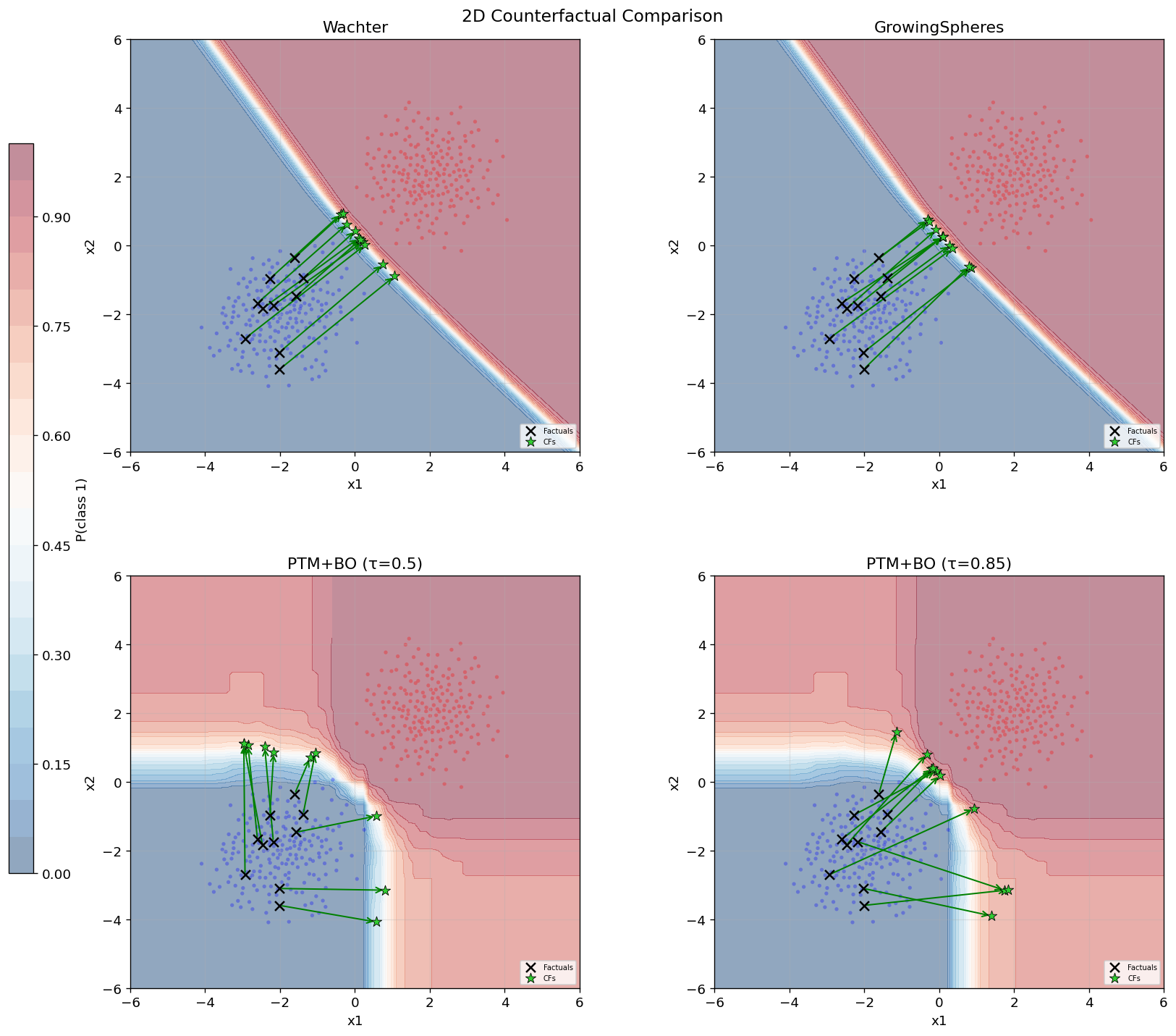}
  \caption{2D Counterfactual Comparison}
  \label{fig:2d}
\end{figure}

\begin{figure}[!ht]
  \centering
  \includegraphics[width=0.8\textwidth]{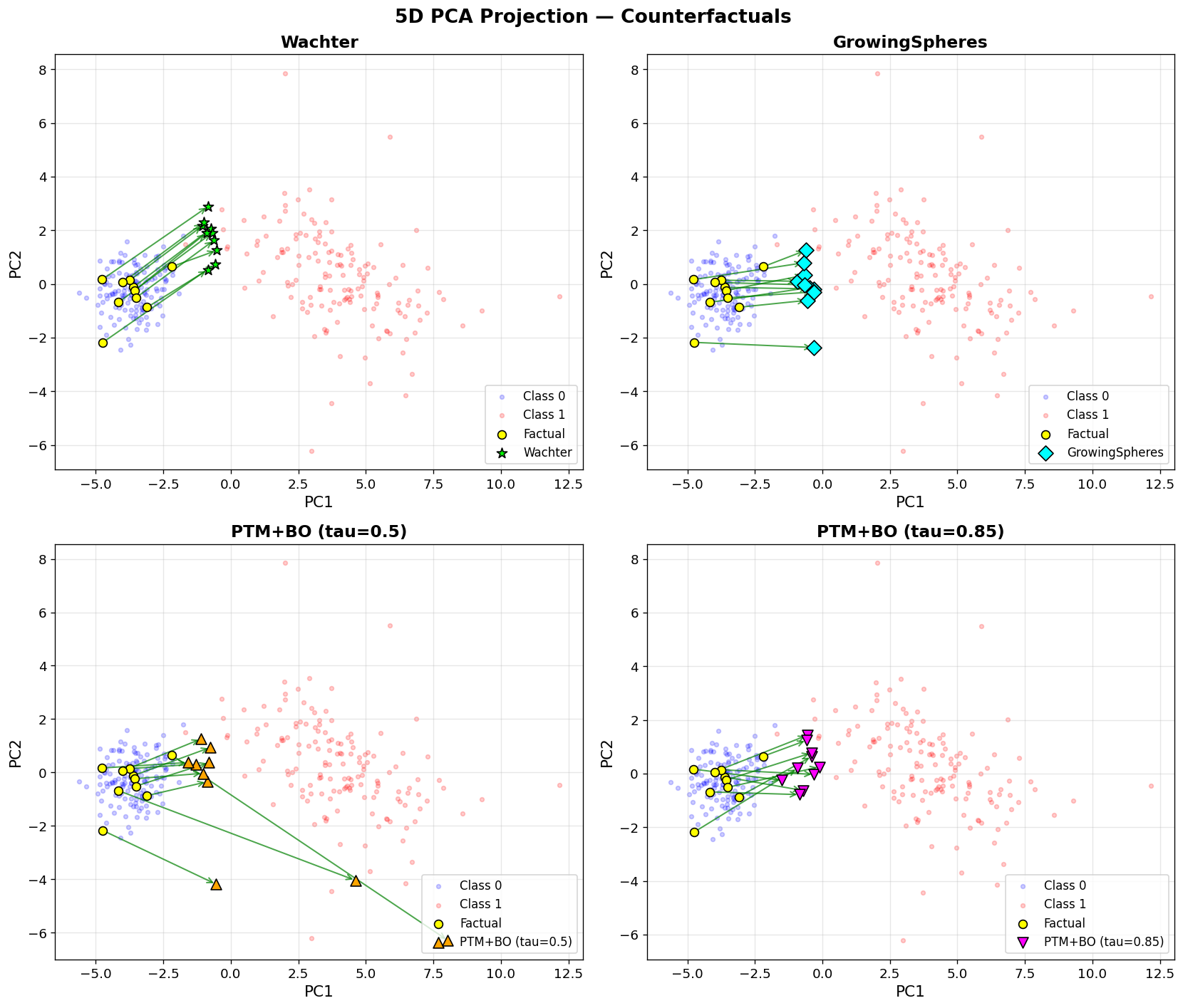}
  \caption{5D PCA Projection — Counterfactuals}
  \label{fig:5d}
\end{figure}

\begin{figure}[!ht]
  \centering
  \includegraphics[width=0.8\textwidth]{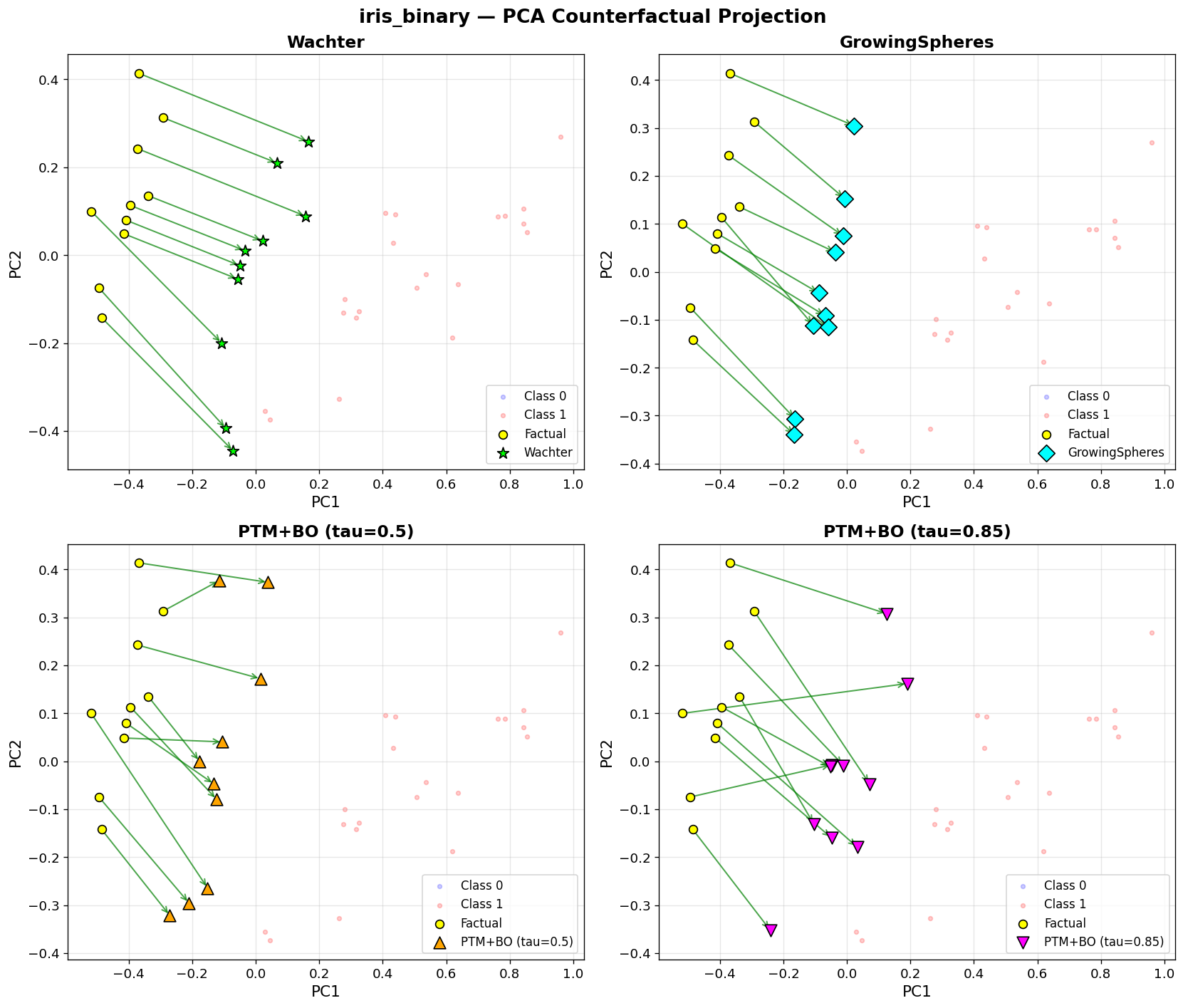}
  \caption{Iris Dataset PCA Projection — Counterfactuals}
  \label{fig:iris}
\end{figure}

\begin{figure}[!ht]
  \centering
  \includegraphics[width=0.8\textwidth]{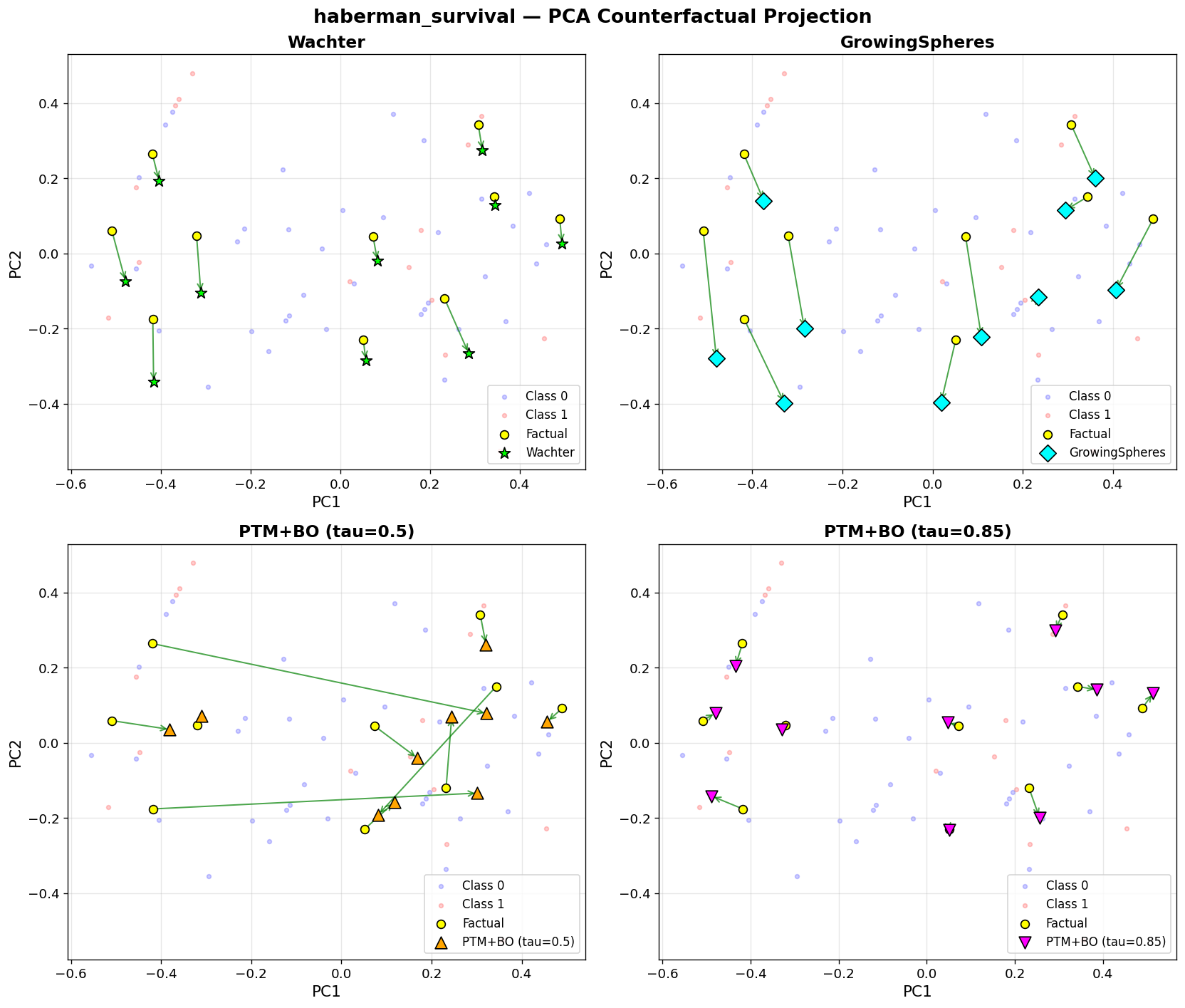}
  \caption{Haberman Survival Dataset PCA Projection — Counterfactuals}
  \label{fig:haberman}
\end{figure}

\paragraph{Synthetic 2D decision boundaries (Figure~\ref{fig:2d}).}
The 2D boundary plots overlay each method's counterfactuals on the classifier's predicted probability surface.
Wachter and GrowingSpheres both move factuals just past the ANN decision boundary, producing counterfactuals that lie on the transition region where $P(\text{class }1)\approx 0.5$; correspondingly, their confidence scores are modest (${\approx}0.56$ and ${\approx}0.52$).
The PTM's learned boundary is visibly axis-aligned. PTM+BO ($\tau{=}0.50$) generates the shortest displacements overall (L2$\,{=}\,0.30$), with counterfactuals landing near the PTM boundary, while PTM+BO ($\tau{=}0.85$) pushes counterfactuals further into the high-confidence red region, achieving perfect robustness (1.00) at the a slightly higher cost.

\paragraph{Synthetic 5D PCA projection (Figure~\ref{fig:5d}).}
In the PCA projection of the five-dimensional setting, Wachter counterfactuals cluster tightly near the class boundary along PC1, consistent with its gradient-descent strategy that stops as soon as the prediction flips.
GrowingSpheres produces more dispersed counterfactuals with competitive L1/L2 distances (L2$\,{=}\,0.29$) and high robustness (0.93).
PTM+BO ($\tau{=}0.50$) counterfactuals spread across the boundary with occasional outliers visible in the lower PC2 range, reflecting the BO's exploration of the search landscape.
PTM+BO ($\tau{=}0.85$) counterfactuals are displaced further along PC1 into the class-1 consistent with the threshold requiring the PTM to assign $P(\text{class }1)\geq 0.85$, which ensures perfect robustness.

\paragraph{Iris PCA projection (Figure~\ref{fig:iris}).} 
On the Iris dataset, the two classes are well separated in PCA space. Wachter counterfactuals lie deep into the class-1 (L1$\,{=}\,0.91$, Confidence$\,{=}\,0.77$), whereas GrowingSpheres places its counterfactuals closer to the boundary but still within the class-1 region. PTM+BO ($\tau{=}0.50$) achieves the smallest displacements of all methods (L2$\,{=}\,0.28$), with counterfactuals clustering at the edge of the class-1 region along PC1, confirming that the proximity-focused threshold finds near-boundary recourse.
PTM+BO ($\tau{=}0.85$) counterfactuals move further into the class-1 cluster attaining the highest confidence.

\paragraph{Haberman Survival PCA projection (Figure~\ref{fig:haberman}).}
The Haberman dataset exhibits substantial class overlap in PCA space.
Wachter and GrowingSpheres both produce counterfactuals that remain scattered near the factuals, reflecting the difficulty of crossing a blurred decision boundary (Confidence$\,{\approx}\,0.50$ for both).
PTM+BO ($\tau{=}0.50$) counterfactuals show moderate displacements with some arrows pointing into ambiguous regions, yet already achieve perfect robustness. PTM+BO ($\tau{=}0.85$) stands out: its counterfactuals are tightly grouped in
the class-1-dominant area of the PCA plot with visibly low cost (L2$\,{=}\,0.1$), the smallest displacement across all methods and datasets, while simultaneously achieving the highest confidence (0.92) and perfect robustness. 

%\newpage
%\input{checklist.tex}

\end{document}